\documentclass[review,onefignum,onetabnum]{siamart171218}

\usepackage{amsfonts}
\usepackage{algorithm, algorithmic}

\title{Online Bandit Linear Optimization: A Study}
\author{V. Mullachery and S. Tiwari}

\begin{document}

\maketitle

\begin{abstract}
  This article introduces the concepts around Online Bandit Linear Optimization and explores an efficient setup called SCRiBLe (Self-Concordant Regularization in Bandit Learning) created by Abernethy et. al.\cite{abernethy}. The SCRiBLe setup and algorithm yield a $O(\sqrt{T})$ regret bound and polynomial run time complexity bound on the dimension of the input space. In this article we build up to the bandit linear optimization case and study SCRiBLe.
\end{abstract}


\section{Introduction}
Online learning setting involves making predictions at every time step. The learner makes a prediction of the label based on the observed input. And then the environment reveals a loss or cost to the learner. Based on this the learner attempts to ameliorate his next prediction and simultenously reduce the cumulative cost. In a bandit setting the learner does not know of the loss function at each time step, only the loss value $\in \mathcal{R}$. Further, this setting does not make any distributional assumption about the actual input data.

SCRiBLe is a learning algorithm that was designed for such scenarios and we discuss it in this article. We describe its algorithmic steps as well as the prerequisites for it's usage. We enumerate and comment about the mathematical underpinnings of this technique. The main theoretical results are in \cref{sec:main}, and the SCRiBLe algorithm is described in \cref{sec:alg}

\section{On-line Convex Optimization Recap}
$\mathcal{K}$ is a convex compact set and $\mathbf{f}_t$ is a convex function defined on $\mathcal{K}$, chosen by the environment in an obvlivious fashion ahead of the player's choices. At each time step, the player predicts a vector $\mathbf{x}_t \in \mathcal{K}$. The player is revealed the loss function $\mathbf{f}_t$. Loss of learner is then computed as $\mathbf{f}_t (\mathbf{x}_t)$. Regret of the player's learning algorithm $\mathcal{A}$ is
    $$R_T(\mathcal{A}) = \sum\limits_{t = 1}^T \mathbf{f}_t (\mathbf{x}_t) - \inf\limits_{\mathbf{x} \in \mathcal{K}} \sum\limits_{t = 1}^T \mathbf{x}_t (\mathbf{x})$$
    Under mild and intuitively obvious assumptions, $O(\sqrt{T})$ regret guarantees are possible, in this scenario.

\section{Bandit On-line Convex Optimization}
Bandit Convex Optimization differs from Online Convex Optimization due to the lack of feedback available to the learner. The learner receives the loss scalar value, $\mathbf{f}_t(\mathbf{x}_t)$, but does not receive the function $\mathbf{f}_t$. Using the same formulation of regret as earlier, we now see that it is much harder to attain $O(\sqrt{T})$ guarantee. 

\section{Bandit Linear Optimization}
\subsection{Problem Setting}
This situation can now be contrasted to the previous and the convex functions $\mathbf{f}_t$ that the environment chooses are linear. If we assume $\mathcal{K} \subset \mathbb{R}^n$, then a linear function is nothing but a vector in $\mathbb{R}^n$ as well. We notice that the environment chooses $\mathbf{f}_t \in \mathbb{R}^n$ and that the learner receives loss $\mathbf{f}_t^T\mathbf{x}_t$. We place the condition that $|\mathbf{f}_t^T\mathbf{x}_t| \leq 1$ $\forall \mathbf{x}_t \in \mathcal{K}$. This is a practical and realistic assumption, which allows for bounds as shown here. Without this condition, the regret can be made to grow linearly.

\subsection{An example}
A canonical motivation for Bandit Linear Optimization is the On-line Shortest Path problem, in which $G = (V, E)$ is a graph, with $s, t \in V$ as the source-sink pair, and with $|E| = n$, the number of edges. The learner is looking for a path from $s$ to $t$, and the environment chooses the time required to traverse each edge. At each time step, the learner is only given the final travel time and no other information. The set of all possible paths may be exponential in $n$, and thus, while this is technically a multi-armed bandit problem, there are far too many arms. The problem is better formulated in $\mathbb{R}^{|E|} = \mathbb{R}^n$. A path is thus a vector in $\mathbb{R}^n$ with each coordinate either $0$ or $1$. Let $\mathcal{K}$ be the convex hull of all possible path vectors. This is well known to be the $\textit{set of all flows}$ in a Graph. Our assumption, $\mathbf{f}_t^T \mathbf{x}_t \le 1$, corresponds to the requirement that no path take longer than $1$ time unit to traverse.

\subsection{Probabilistic Perturbation}
In general for Bandit Linear Optimization problems, the learner must act probabilistic in order to hedge against the possibility of the environment being adversarial in it's choices of $\mathbf{f}_t$. For instance, Follow The Leader algorithm suffers linear regret at the hands of an adaptive adversary. 
Usuallu, the player evaluates an optimum point $\mathbf{x}_t$, and then plays a pertrubation of it, $\mathbf{y}_t$. This perturbation is employed to evaluate $\tilde{\mathbf{f}_t}$, an estimate of $\mathbf{f}_t$. The learner proceeds to evaluate $\mathbf{x}_{t+1}$ using this, and prior estimates. $\tilde{\mathbf{f}_t}$ are single-point gradient estimates, and are random variables such that:
$\mathbb{E}[\tilde{\mathbf{f}_t}] = \mathbf{f}_t$

Since the (possibly adversarial) environment's $\mathbf{f}_t$ are not only convex but also linear, these estimates work well. This idea can be seen to transform the bandit setting into a full information setting.

\subsection{Projected Gradient Descent}
Projection methods, that are instances of mirror descent, require one to work with a $\mathcal{K}_\delta \subset \mathcal{K}$. The points $\mathbf{x}_t$ are necessarily in $\mathcal{K}_\delta$ and a perturbation of size $\delta$ results in $\mathbf{y}_t$ in $\mathcal{K}$. With $\delta, \eta$ as hyper-parameters, BanditPGD obtains the regret:
$R_T(\text{BanditPGD}) = \frac{C_1}{\eta} + \frac{C_2\eta T}{\delta^2} + C_3T\delta$
for constants $C_1, C_2, C_3$.

\subsection{What could go wrong}
Since the optimal point $\mathbf{x}^*$ is guaranteed to be on the boundary of $\mathcal{K}$ (due to the linearity of $\mathbf{f}_t$), and since projection requires $\mathbf{x}_t \in \mathcal{K}_\delta$, we suffer a $O(\delta T)$ cost by staying away from the boundary. The $O(\frac{T}{\delta^2})$ appears due to the $\mathbb{E}[\tilde{\mathbf{f}_t^T}(\mathbf{x}_t - \mathbf{x}^*)]$ term. Using Cauchy-Schwarz to bound this, one encounters $\mathbb{E} |\tilde{\mathbf{f}_t}|^2 = O(\frac{1}{\delta^2}) $. The result is a suboptimal $O(T^{\frac{3}{4}})$ regret guarantee.

Dani, Hayes and Kakade \cite{dani} utilize geometric hedge algorithm and bounded decision set to achieve $O(n^{3/2}\sqrt{T})$ regret. While regret is optimal in time, it does not admit an efficient implementation, due to computational complexity dependence on $poly(n)$. Such computational costs are impractical for applications to Online Shortest Path - recall that $n$ represents the number of edges in the graph.

\subsection{BanditFTRL}
Bandit Follow The Regularized Leader (BanditFTRL), addresses these issues by using superior barrier functions to deal with the geometry of $\mathcal{K}$, instead of projections. Since the cost of projection is removed from regret analysis, expected regret is $O(\sqrt{T})$ and not $O(T^{\frac{3}{4}})$.  In this setup, the most computationally challenging task at each round is finding the $argmin$ of a strongly convex function. If one were to modify the algorithm and replace the $argmin$ step by a single iteration of the Damped Newton method, the expected regret provably enjoys the same asymptotics. This allows for an implementation of a solution to the Online Shortest Path problem with $O^*(\sqrt{T})$ computational complexity.

\section{Self-Concordant Regularization in Bandit Learning (SCRiBLe)}\label{sec:main}

\subsection{Ingredients}
\begin{itemize}
\item A regularization function $\mathcal{R} : \mathcal{K} \rightarrow \mathbb{R}$. Specific regularization functions yield well-known algorithms. If $\mathcal{R} : \mathbf{x} \mapsto \lVert \mathbf{x} \rVert^2$ is chosen, FTRL results in Online Gradient Descent algorithm. Similarly, choosing entropy function gives the Exponentiated Gradient algorithm. Thus, the first ingredient of the algorithm is an appropriate $\mathcal{R}$. As will be seen shortly, a self-concordant barrier function is used as the regularizer.

\item To transform the bandit problem to the full-information case, we use single-point gradient estimates $\tilde{\mathbf{f}_t}$. As observed earlier with BanditPGD, $\mathbb{E}\lVert \tilde{\mathbf{f}_t} \rVert^2$ is the troublesome term. This variance term is avoided by considering local norms instead of the standard Euclidean norm

\item The final ingredient is a sampling scheme. Like in BanditPGD, the learner never plays $\mathbf{x}_t$, but rather plays a nearby point $\mathbf{y}_t$. BanditPGD samples according to a sphere of radius $\delta$. Instead, we use a sampling scheme based on the Dikin ellipsoid.
    
\end{itemize}

\subsection{Prerequisites}
\begin{itemize}

\item \textbf{Self-Concordant Function}: 
$\mathcal{R} : \text{int}(\mathcal{K}) \rightarrow \mathbb{R}$ is a self-concordant function if for any $\mathbf{h} \in \mathbb{R}^n$, we have:
$|D^3\mathcal{R}(\mathbf{x})[\mathbf{h}, \mathbf{h}, \mathbf{h}]| \leq 2(D^2 \mathcal{R}(\mathbf{x})[\mathbf{h}, \mathbf{h}])^{\frac{3}{2}}$. 
Additionally, $\lim\limits_{\mathbf{x} \rightarrow \delta \mathcal{K}} \mathcal{R}(\mathbf{x}) = \infty$. This compares to Legendre-type functions from Mirror Descent.

\item \textbf{Self-Concordant Barrier}: A Self-Concordant $\mathcal{R}$ is a $\vartheta$-Self-Concordant barrier if $\forall \mathbf{h} \in \mathbb{R}^n$: $|D \mathcal{R} (\mathbf{x})[\mathbf{h}]| \leq \sqrt{\vartheta D^2 \mathcal{R} (\mathbf{x})[\mathbf{h}, \mathbf{h}]}$

\item \textbf{Dikin Ellipsoid}: The Hessian of $\mathcal{R}$ is a positive definite symmetric matrix, so that for any point $\mathbf{x} \in \text{int}(\mathcal{K})$, we can define
$\langle \mathbf{y}_1, \mathbf{y}_2 \rangle_\mathbf{x} = \mathbf{y}_1^T \nabla^2 \mathcal{R}(\mathbf{x}) \mathbf{y}_2$. This inner product gives a $\textit{local norm}$ at $\mathbf{x}$, denoted by $\lVert \rVert_\mathbf{x}$. The Dikin Ellipsoid (of radius $1$) at $\mathbf{x}$ is given by: 
$W_1(\mathbf{x}) = \lbrace \mathbf{y} \in \mathcal{K} : \lVert \mathbf{y} - \mathbf{x} \rVert_\mathbf{x} < 1 \rbrace$. The Dikin ellipsoid is therefore an open ball corresponding to a specific inner product on $\mathbb{R}^n$. Note an important property that for any interior point $\mathbf{x}$ of $\mathcal{K}$, the Dikin ellipsoid  at $\mathbf{x}$ is also contained within the interior of $\mathcal{K}$. This is a consequence of the fact that $\mathcal{R}$ is a $\textit{barrier}$ function - the Dikin ellipsoids gets flatter as one approaches the boundary of the set. Thus, by sampling points in the Dikin ellipsoid, one does not need to perform any projections in the algorithm.

\end{itemize}

\subsection{Main Theorem and Supporting Lemmas}

Here we state the main theorem \cref{thm:scriblemain} of the SCRiBLe algorithm. First we introduce two lemmas that are required to prove the theorem.

\begin{lemma}[Bandit Reduction Lemma]
\label{lem:breducn}
Assume we are given any full information algorithm $\mathcal{A}$ and unbiased sampling and estimating schemes $\textbf{sampler, guesser}$. If we let the associated Bandit algorithm be $\mathcal{A}' = \texttt{BanditReduction}(\mathcal{A}, \textbf{sampler, guesser} ) $, then the expected regret of the randomized algorithm $\mathcal{A}'$ on the fixed sequence $\lbrace \mathbf{f}_t \rbrace$ is equal to the expected regret of the deterministic algorithm $\mathcal{A}$ on the random sequence $\lbrace \tilde{\mathbf{f}_t} \rbrace$.
\begin{displaymath}
\mathbb{E}[\text{Regret}^\mathbf{u}(\mathcal{A}; \mathbf{f}_1, \mathbf{f}_2 \ldots \mathbf{f}_T)] = \mathbb{E}[\text{Regret}^\mathbf{u}(\mathcal{A}'; \tilde{\mathbf{f}}_1, \tilde{\mathbf{f}}_2, \ldots \tilde{\mathbf{f}}_T)]
\end{displaymath}
\end{lemma}

\begin{lemma}[Regret Bound Lemma]
\label{lem:regb}
Assume that $\eta \lVert \mathbf{f}_t \rVert_{\mathbf{x}_t}^* \leq \frac{1}{4}$ and that $\mathcal{R}$ is a Self-Concordant barrier with $\min \mathcal{R}(\mathbf{X}) = 0$. Then for any $\mathbf{u} \in \mathcal{K}$,
\begin{displaymath}
\text{Regret}^\mathbf{u}(\text{FTRL}(\mathcal{R}, \mathbf{f}_{1:t})) \leq 2\eta \sum\limits_{t=1}^T \lVert \mathbf{f}_t \rVert_{\mathbf{X}_t}^{*2} + \eta^{-1}\mathcal{R}(\mathbf{u})
\end{displaymath}
\end{lemma}

\begin{theorem}[SCRiBLe Main]
\label{thm:scriblemain}
Let $\mathcal{K}$ be a compact convex set $\subset \mathbb{R}^n$, and $\mathcal{K}$ be a $\vartheta$-self-concordant barrier on $\mathcal{K}$. Assume $|\mathbf{f}_t^T \mathbf{x} | \leq L$ for any $\mathbf{x} \in \mathcal{K}$ and any $t$. Setting $\eta = \sqrt{\frac{\vartheta \log T}{2n^2L^2T}}$, the regret of SCRiBLe is bounded as 
$ \mathbb{E}[Regret^u(SCRiBLe;\mathbf{f}_1 \cdots \mathbf{f}_T)] \leq nL\sqrt{8 \vartheta T \log T} + 2L$ \text{whenever} $\frac{T}{\log T} > 8 \vartheta$ 
\end{theorem}

\begin{proof}
By Bandit Reduction lemma, we can write the regret as: $\mathbb{E}[Regret^u(\mathcal{A}; \mathbf{f}_1 \cdots \mathbf{f}_T)] = \mathbb{E}[Regret^u(FTRL_{\mathcal{R}};\mathbf{\tilde{f}}_1 \cdots \mathbf{\tilde{f}}_T)]$
Now applying Regret Bound theorem:
\begin{align*}
\eta \left\vert\left\vert \mathbf{f}_t \right \vert \right\vert_{\mathbf{x}_t}^{*} &= \eta \sqrt{ \mathbf{\tilde{f}}_t^T \nabla^{-2} \mathcal{R}(\mathbf{x}_t) \mathbf{\tilde{f}}_t} \\
&= \eta n |\mathbf{f}_t^T \mathbf{y}_t| \sqrt{\lambda_i\mathbf{e}_i \nabla^{-2}\mathcal{R}(\mathbf{x}_t)\mathbf{e}_i} \\
&= \eta n |\mathbf{\tilde{f}}_t \mathbf{y}_t| \\
&\leq \eta n L \\
&\leq n L\sqrt{\frac{\vartheta \log T}{2n^2L^2T}} \\
&\leq \frac{1}{4} \tag{$\frac{T}{\log T} > 8\vartheta$} \\
\end{align*}

\text{Thus,} $\left\vert\left\vert \mathbf{f}_t \right \vert\right\vert_{\mathbf{x}_t}^{*2}  \leq n^2L^2$. \text{So,} $\forall \mathbf{u} \in \mathcal{K}$,

\begin{align*}
\mathbb{E}[Regret^{\mathbf{u}}(FTRL_{\mathcal{R}};\mathbf{\tilde{f}}_1 \cdots \mathbf{\tilde{f}}_T)]
&\leq 2 \eta \mathbb{E}[\sum_{t=1}^{T} \left\vert\left\vert \mathbf{\tilde{f}}_t\right \vert \right \vert_{\mathbf{x}_t}^{*2}] + \eta^{-1}\mathcal{R}(\mathbf{u})\\
&\leq 2 \eta n^2 L^2T + \eta^{-1}\mathcal{R}(\mathbf{u})
\end{align*}
\text{If } $\pi_{\mathbf{x}_1}(\mathbf{u}) \leq 1 - 1/T$, $\mathcal{R}(\mathbf{u}) \leq \vartheta \log T$. Otherwise, define $\mathbf{u}' = (1- 1/T)\mathbf{u} + (1/T)\mathbf{x}_1 $. And now: 
\begin{align*}
Regret^{\mathbf{u}}(\mathcal{A};\mathbf{f}_{1:T}) &= Regret^{\mathbf{u'}}(\mathcal{A}; \mathbf{f}_{1:T}) + \sum_{t=1}^{T} \mathbf{f}_t^T(\mathbf{u}' - \mathbf{u})\\
&= Regret^{\mathbf{u'}}(\mathcal{A}; \mathbf{f}_{1:T}) + \frac{1}{T}\sum_{t=1}^{T} \mathbf{f}_t^T(\mathbf{x}_1 - \mathbf{u})\\
&= 2 \eta n^2 L^2T + \eta^{-1}\mathcal{R}(\mathbf{u}') + 2L\\
&\leq 2 \eta n^2 L^2T + \vartheta \eta^{-1}\log T + 2L
\end{align*}
Plugging in the value for $\eta$ completes the proof.
\end{proof}

\section{Algorithm}
\label{sec:alg}
At each iteration the algorithm performs an eigen decomposition of the Hessian of the Self-concordant barrier function at $\mathbf{x}_t$. This is the most computation intense step in this scheme. The update to $\mathbf{x}_t$ in step $9$ occurs in a single Newton decrement step.
\begin{algorithm}[H]
\begin{algorithmic}[1]
\STATE $Input: \eta > 0$ $ \vartheta$-self-concordant barrier $\mathcal{R}$ 
\STATE $Let $ $\mathbf{x}_1 = \text{argmin}_{\mathbf{x} \in \mathcal{K}}[\mathcal{R}(\mathbf{x})]$
\FOR{$t=1$ to $T$}
\STATE $Compute$ $\{ \mathbf{e_1} \cdots \mathbf{e_n} \}, \{ \lambda_1 \cdots \lambda_n \} $ eigenvectors and eigenvalues of $ \nabla^{2} \mathcal{R}(\mathbf{x}_{t})$
\STATE $Choose$ $ i$ uniformly at random from $ \{ 1 \cdots n\}$ and $\varepsilon = \pm 1$ with probability $1/2$
\STATE $Predict $ $ \mathbf{y}_t = \mathbf{x}_t + \varepsilon \lambda_i^{-1/2}\mathbf{e}_i$
\STATE $Observe$ the cost $\mathbf{f}_t^{T}\mathbf{y}_t \in \mathbb{R}$
\STATE $Define $ $ \mathbf{\tilde{f}}_t :=  \eta (\mathbf{f}_t^T\mathbf{y}_t)\varepsilon \lambda_i^{1/2} \mathbf{e}_i$
\STATE $Update $
$\mathbf{x}_{t+1} = \text{argmin}_{\mathbf{x} \in \mathcal{K}} [ \eta \sum_{s=1}^{t} \mathbf{\tilde{f}}_s^{T} \mathbf{x} + \mathcal{R}(\mathbf{x}) ]$
\ENDFOR
\end{algorithmic}
\caption{SCRiBLe (Self-Concordant Regularization in Bandit Learning)}
\end{algorithm}
\subsection{Explore-exploit trade-off}
As $\eta \sum\limits_{i = 1}^t \tilde{\mathbf{f}}_t$ grows in norm, the effect of the Self-Concordant barrier diminishes, since $\mathbf{x}_{t+1}$ is chosen as the $argmin$ of:
$\eta \sum\limits_{i = 1}^t \tilde{\mathbf{f}}_t\mathbf{x} + \mathcal{R}(x)$. Due to the geometry of the Dikin ellipsoid, as $\mathbf{x}_t$ approaches a boundary of $\mathcal{K}$, the Dikin ellipsoid at $\mathbf{x}_t$ becomes flatter and orients itself along the edge. Consequently, the sampled $\mathbf{y}_t$ is less likely to deviate from $\mathbf{x}_t$ in this direction.

\bibliographystyle{siamplain}
\bibliography{references}

\end{document}